\documentclass[letterpaper, 10pt, conference]{ieeeconf}   

\IEEEoverridecommandlockouts                              

\usepackage{import} 
\usepackage[pdftex]{graphicx}
\usepackage{sidecap}
\usepackage{cite}
\usepackage{verbatim}		
\usepackage{amsmath}
\usepackage{amssymb}

\usepackage{amsthm}
\usepackage{mathtools}	
\usepackage{grffile}	
\usepackage[tight,footnotesize]{subfigure}
\usepackage{microtype} 
\usepackage{color}
\usepackage{url}
\usepackage[ruled,vlined,linesnumbered]{algorithm2e}
\usepackage{bm} 
\usepackage{comment}
\usepackage{arydshln}
\let\labelindent\relax
\usepackage{enumitem}

\usepackage{adjustbox}
\usepackage{tikz}
\usetikzlibrary{shadings, patterns, angles, quotes, arrows.meta, shapes, decorations.pathmorphing, decorations.shapes, decorations.text, automata, positioning}
\usetikzlibrary{calc,intersections,arrows.meta}
\usepackage{pgfplots}

\usepackage{hyperref}
\hypersetup{
	colorlinks=true,
	linkcolor=blue,
	citecolor=red,
}

	\tikzset{
	pil/.style={
		->,
		thick,
		shorten <=2pt,
		shorten >=2pt,}
}

\theoremstyle{remark}
\newtheorem{remarkx}{Remark}
\newenvironment{remark}
{\pushQED{\qed}\remarkx}
{\popQED\endremarkx}

\theoremstyle{definition}
\newtheorem{defn}{Definition}
\newtheorem{assump}{Assumption}
\newtheorem{problem}{Problem}
\newtheorem*{problem*}{Problem}

\theoremstyle{plain}
\newtheorem{theorem}{Theorem}
\newtheorem{lemma}{Lemma}

\newtheorem{prop}{Proposition}

\newcommand{\red}[1]{{\leavevmode\color{red}#1}}

\newcommand{\green}[1]{{\leavevmode\color{green}#1}}

\newcommand{\cmmnt}[1]{}

\newcommand{\scalemath}[2]{\scalebox{#1}{\mbox{\ensuremath{\displaystyle #2}}}}

\title{\LARGE \bf Behavioral-based circular formation control for robot swarms}

\author{Jesús Bautista, Héctor García de Marina 
	\thanks{J. Bautista and H.G. de Marina are with the Department of Computer Engineering, Automation, and Robotics, University of Granada (UGR-CITIC), Spain. This work is supported by the ERC Starting Grant \emph{iSwarm} 101076091 and the RYC2020-030090-I grant from the Spanish Ministry of Science.}
}

\begin{document}

\maketitle
\thispagestyle{empty}
\pagestyle{empty}

\begin{abstract}
This paper focuses on coordinating a robot swarm orbiting a convex path without collisions among the individuals. The individual robots lack braking capabilities and can only adjust their courses while maintaining their constant but different speeds. Instead of controlling the spatial relations between the robots, our formation control algorithm aims to deploy a dense robot swarm that mimics the behavior of \emph{tornado} schooling fish. To achieve this objective safely, we employ a combination of a scalable overtaking rule, a guiding vector field, and a control barrier function with an adaptive radius to facilitate smooth overtakes. The decision-making process of the robots is distributed, relying only on local information. Practical applications include defensive structures or escorting missions with the added resiliency of a swarm without a centralized command. We provide a rigorous analysis of the proposed strategy and validate its effectiveness through numerical simulations involving a high density of unicycles.
\end{abstract}

%


\section{Introduction} 
\label{sec: intro}

Formation control, one of the most actively studied topics within multi-robot systems, aims to coordinate multiple robots spatially given the practical applications \cite{yang2018grand}. The idea of formation can be understood in different ways. One way is by looking at how robots are positioned in relation to each other, like how far apart they are or the angles between them \cite{oh2015survey}. Another way, called the behavioral approach, does not worry about exact positions but keeps cohesion and ensures robots do not crash into each other \cite{balch1998behavior}. Although the bio-inspired and heuristic strategies fit more into the latter approach \cite{reynolds1987flocks}, there also exist rigorous behavioral-based studies \cite{olfati2006flocking}.

In this paper, we present and analyze rigorously a circular formation algorithm that fits more into the behavioral-based approach, in the sense that we aim at the order that \emph{emerges} from having a large number of robots following the same closed trajectory rather than the classical pursuit and organization in the circle \cite{marshall2006pursuit}. In particular, we put together a high density of unicycles whose mission is to follow the same convex path while traveling at a constant but different speeds. Our proposed strategy to avoid collisions in combination with a guiding vector field to follow the path results into an intricate coordination reminiscent of the emergent behavior observed in a school of fish. Using guiding vector fields that spiral towards a desired circle imitates how fish closely track their neighbors' stream with their lateral line organ for \emph{tornado} schooling \cite{bone2008biology}. Another interesting behavior occurs when we have two groups of robots following the same path but in opposite direction. In this case, robots organize themselves \emph{drawing} a dynamic boundary around the path in order to avoid frontal collisions between the two groups.

Independently of the chosen algorithm to follow the path, one effective strategy to avoid collisions between robots is to employ \emph{Control Barrier Functions} (CBFs) \cite{cbf_theory}, a mathematical tool that guarantees robots to be away from unsafe sets for the relative positions. To the best of our knowledge, the literature demonstrates CBFs for mobile robots that can change their speeds \cite{wang2017safety, borrmann2015control,chen2017obstacle,panagou2015distributed, wang2016multi, chen2020guaranteed}, i.e., braking to avoid an imminent collision. However, demanding drastic changes in the speeds of vehicles such as fixed-wing UAVs compromise their energy efficiency and stability regarding keeping a leveled altitude or even stalling and crashing, but this time into the ground.

Unfortunately, we observe that the common employed \emph{Collision Cone Control Barrier Function} (C3BF) \cite{c3bf_arigraft,c3bf_UAV,c3bf_car,c3bfs_missile} is not enough to guarantee a collision-free environment when unicycles cannot brake. For example, one robot might have to take an impossible decision to avoid two robots, i.e., to turn to the right and to the left simultaneously. To solve this problem, we present a scalable overtaking rule, that in combination with an adaptive-collision radius, and the guiding vector field presented in \cite{kapitanyuk2017guiding}, ensures a collision-free environment for unicycles that follow the same path with constant but different speeds. In particular, the usage of the guiding vector field in \cite{de2017guidance, kapitanyuk2017guiding} facilitates one mild condition so that we can extend the results to convex closed paths. In addition, the decision making process is completely distributed, i.e., there is no centralized computation with global information, but an individual robot can figure out a safe action to avoid a collision on its own with local information and, more importantly, it will not harm others' decisions concerning their safety as it is crucial to scale up robot swarms \cite{wang2017safety, borrmann2015control}.

We organize the paper as follows. In Section \ref{sec: pre} we introduce the necessary preliminaries on the employed guiding vector field and control barrier functions to formulate our problem and objective formally. Next, in Section \ref{sec: collision}, we introduce and analyze our collision avoidance strategy. In Section \ref{sec: sim}, we validate our results on the behavioral-based circular formation control through simulations involving dense populations of robots. Finally, Section \ref{sec: con} ends with some conclusions.

\section{Preliminaries and problem formulation} \label{sec: pre}

\subsection{Notation and the unicycle model with constant speed}
Let us define $\hat x := \frac{x}{||x||}$, the rotational matrixes $E := R(-\pi/2) = \left [ \begin{smallmatrix}
      0 & 1 \\
      -1 &  0
  \end{smallmatrix} \right ]$, and a continuous function $\kappa : (-c,d) \rightarrow (-\infty,\infty)$ that belong to (extended) \textit{class} $\mathcal{K}$, i.e., it is strictly increasing and $\kappa(0) = 0$. We consider a robot swarm of $N  \in \mathbb{N} \geq 2$ unicycles labeled by $i \in \mathcal{N} := \{1,2, ... ,N\}$, where $\mathcal{N}$ is the set of robots. The dynamic state of each robot is described by $q_i \in \mathcal{D}_i \subset \mathbb{R}^2\times \mathbb{S}$, where $\mathbb{S}:=(-\pi,\pi]$, and $\mathcal{D}_i$ is the set of \textit{reachable states} of robot $i$. When the context is clear, for the sake of clarity, we drop the subindex $i$ to refer just to a single robot. The dynamics of the unicycle are expressed as 

\begin{equation} \label{eq: kinematics}
    \dot q_i =
    \left [
      \begin{array}{c}
         \dot p_i^X    \\
         \dot p_i^Y \\
         \dot \theta_i
      \end{array}
    \right ] = 
    s_i
    \left [
      \begin{array}{cc}
         \cos \theta_i\\
         \sin \theta_i\\
         0
      \end{array}
    \right ]
    +
    \left [
      \begin{array}{cc}
         0\\
         0\\
         1
      \end{array}
    \right ]
    \omega_i,
\end{equation}
where, $[p_i^X\; p_i^Y]^T =: p_i \in \mathbb{R}^2$ is the absolute position of the robot $i$, $s_i \in \mathbb{R}^+$ its constant linear speed and $w_i \in \mathbb{R}$ its angular speed with $\theta_i \in \mathbb{S}$ being the heading angle.

Given a pair of robots $(i,j)$, the relative position vector centered on $i$ is defined as $p_{ij} := p_j - p_i = \left [ \begin{smallmatrix}
      p_j^X - p_i^X \\
      p_j^Y - p_i^Y
  \end{smallmatrix} \right ]$ and the relative angle from $i$ as $\theta_{ij} := \theta_j - \theta_i$. Then, considering \eqref{eq: kinematics}, the first time derivative of $p_{ij}$ yields to the relative velocity, i.e., $v_{ij} := \dot p_{ij} = v_j - v_i$, so the dynamical model of a relative position with absolute heading coordinates is given by

\begin{align} \label{eq: rel_kinematics}
    \scalemath{0.78}{
    \dot q_{ij} =
    \left [
      \begin{array}{c}
         \dot p_{ij}^X \\
         \dot p_{ij}^Y \\
         \dot \theta_{i}\\
         \dot \theta_{j}
      \end{array}
    \right ] = 
    \underbrace{
    \left [
    \begin{array}{cc}
        s_j\cos\theta_j - s_i\cos\theta_i\\
        s_j\sin\theta_j - s_i\sin\theta_i\\
        0\\
        0
    \end{array}
    \right ]
    }_{f(q_{ij})}
    +
    \underbrace{
    \left [
      \begin{array}{cc}
         0 & 0\\
         0 & 0\\
         1 & 0\\
         0 & 1
      \end{array}
    \right ]
    }_{g(q_{ij})}
    \left [
      \begin{array}{cc}
         \omega_i\\
         \omega_j\\
      \end{array}
    \right ]
    },
\end{align}
where $f:\mathbb{R}^2\times\mathbb{T}^2 \rightarrow \mathbb{R}^4$ and $g:\mathbb{R}^{2}\times\mathbb{T}^2 \rightarrow \mathbb{R}^{4 \times 2}$ are locally Lipschitz functions, with $\mathbb{T}^2 := \mathbb{S}\times\mathbb{S}$, and $\omega_i$ and $\omega_j$ are the control laws $u_{i}$ and $u_{j}$, respectively. Both control laws must be in $\mathcal{U} \subset \mathbb{R}$ due to inherent physical limitations.


\subsection{Guiding vector fields}

The authors in \cite{kapitanyuk2017guiding} present a guiding vector field for the path following of smooth 2D curves. In particular, it is a vector field that guides only in direction and not in speed, making it suitable for vehicles such as fixed-wing aircraft \cite{de2017guidance} whose ground speed depends on the wind. The path to be followed $\mathcal{P} \subset \mathbb{R}^2$ is defined as a \textit{one-dimensional connected submanifold} corresponding to the zero-level set of an implicit function, or \emph{desired path}, $\mathcal{P} := \{ p : \varphi(p) = 0\}$, where $\varphi: \mathbb{R}^2 \rightarrow \mathbb{R}$ is twice continuously differentiable and regular in a neighborhood of $\mathcal{P}$. This representation allows to cover the plane $\mathbb{R}^2$ by $\varphi(p) = c \in\mathbb{R}$, so one can employ these level sets as a metric for the error between the robot and the desired path $\mathcal{P}$, i.e., $e(p) := \varphi(p) \in \mathbb{R}$. Considering this metric, it is shown in \cite{kapitanyuk2017guiding} that the vector field 
\begin{equation}
\dot p_d(p) := \tau(p) - k_e e(p) n(p),
\label{eq: gvf}
\end{equation}
guides the robots to converge and travel over $\mathcal{P}$. The gain $k_e \in \mathbb{R}^+$ tunes the convergence aggressiveness towards $\mathcal{P}$, where $n(p) := \nabla\varphi(p)$ and $\tau(p) := En(p)$ are the normal and tangent vectors, respectively, to the curve corresponding to the level set $\varphi(p)$. It is shown in \cite{de2017guidance} how to design a heading controller $u_{\text{ref}}$ for a fixed-wing aircraft with constant speed such that it aligns its velocity $\dot p(t)$ to the guiding vector field $\dot p_d(p)$, where a gain $k_d\in\mathbb{R}^+$ tunes the aggressiveness of such a convergence. For a closed path $\mathcal{P}$, we note that a small $k_e$ will make the unicycle to orbit almost in parallel to $\mathcal{P}$, e.g., a spiral towards the desired circular path.

\subsection{Control Barrier Functions}

The CBF is a mathematical tool that prevents the state of a system from entering an unsafe set. Let us define the following closed set $\mathcal{C}$ (\textit{safe set}):
\begin{equation}
\begin{rcases}
    \mathcal{C} &:= \{x \in \mathcal{D} \subset \mathbb{R}^n \; : \; h(x) \geq 0\} \\
    \partial\mathcal{C} &:= \{x \in \mathcal{D} \subset \mathbb{R}^n \; : \; h(x) = 0\} \\
    \text{Int}(\mathcal{C}) &:= \{x \in \mathcal{D} \subset \mathbb{R}^n \; : \; h(x) > 0\} \end{rcases},  \label{eq: safe_set}
\end{equation}
where $x\in\mathbb{R}^n$ is the system state, and $h: \mathcal{D} \rightarrow \mathbb{R}$ is a continuously differentiable function.

\begin{defn} [Safety \cite{cbf_theory}] We say that a given system is \emph{safe} if $\mathcal{C}$ is forward invariant, i.e., $\forall x(0) \in \mathcal{C}$, $x(t) \in \mathcal{C} \; \forall t \geq 0$. 
\end{defn}

\begin{defn} [Valid CBF \cite{cbf_theory}] 
\label{def: validcbf}
    Given a safe closed set $\mathcal{C}$ as in \eqref{eq: safe_set}, the function $h: \mathcal{D} \rightarrow \mathbb{R}$ with $\frac{\partial h}{\partial x}(x) \neq 0, \forall x \in \partial\mathcal{C}$, is a \emph{valid CBF} if there exists an (extended) class $\mathcal{K}$ function $\kappa$ such that $\forall x \in \mathcal{C}$ the following \textit{safety condition} is satisfied:
    \begin{equation} \label{eq: CBF_cond}
    \underbrace{\text{sup}}_{u\in\mathcal{U}}\left[\underbrace{\mathcal{L}_f h(x) + \mathcal{L}_g h(x)u}_{\dot h(x)} + \kappa(h(x))\right] \geq 0, 
    \end{equation}
    where $\mathcal{L}_f h(x) := \frac{\partial h}{\partial x}f(x)$ and $\mathcal{L}_g h(x) := \frac{\partial h}{\partial x}g(x)$.
\end{defn}
For the sake of conciseness throughout the paper, let us define the inner expression of (\ref{eq: CBF_cond}) only evaluated with $u_{\text{ref}}$ as
\begin{equation}
 \label{eq: psi}
    \Psi(x) := \dot h(x,u_\text{ref}(x)) + \kappa(h(x)).
\end{equation} 
Later, we will suggest a set $\mathcal{C}$ linked to the robots' relative states that does not imply a risk of collision. Hence, our problem, to be formalized later, is to find a control law that guarantees the safety, or no collisions, of a system consisting of an arbitrary number of robots that follow a common path $\mathcal{P}$ and with relative dynamics \eqref{eq: rel_kinematics}. To do so, we will use a valid CBF. Now, let us consider the nominal input $u_\text{ref}$, i.e., the angular velocity that aligns robot's $v_i$ with \eqref{eq: gvf} \cite{de2017guidance}. We can modify $u_\text{ref}$ in a minimal way so that the inner expression of \eqref{eq: CBF_cond} is positive via the \emph{Quadratic Programming} (QP) problem
\begin{equation} \label{eq: CBF-QP}
    \begin{array}{rl}
        u^*(x)  &= \underset{u \, \in \, \mathcal{U}}{\text{min}} \; \|u - u_\text{ref}(x)\|^2 \\
        \text{s.t.} & \quad \mathcal{L}_f h(x) + {\mathcal{L}_g h(x)}^T u + \kappa(h(x)) \geq 0,
    \end{array}
\end{equation}
as in \cite{cbf_qp_control}. The function $\kappa$ modulates the \textit{aggressiveness} to prevent the system state from getting out of the safe set $\mathcal{C}$, with $\kappa(h(x)) = 0$ representing the most aggressive scenario, i.e., the most significant modification of the nominal $u_\text{ref}$.

We will exploit later the following particular case. If there is only one condition, the QP problem (\refeq{eq: CBF-QP}) could be feasible with the solution $u^*(x) = u_\text{ref}(x) + u_\text{safe}(x)$ with
\begin{equation} \label{eq: u_safe}
	u_\text{safe}(x) = 
    \begin{cases}
        \begin{array}{lr}
            0 
            & \text{if} \; \Psi(x) \geq 0\\
            - \, \frac{\mathcal{L}_g h(x)^T\Psi(x)}{\mathcal{L}_g h(x)^T\mathcal{L}_g h(x)} 
            & \text{if} \; \Psi(x) < 0
        \end{array}
    \end{cases}\hspace{-0.5cm}.
\end{equation}
Note that when $\|\mathcal{L}_g h(x)\| \rightarrow 0$ we approach to a singularity where $\|u_\text{safe}(x)\| \rightarrow \infty$. 

We say that the pair robots $(i,j)$ do not \textit{collaborate} when robot $j$ is seen as an \emph{obstacle} so that $\omega_j$ is fixed and robot $i$ has to do \emph{all the work} modifying $\omega_i$; therefore, the QP problem is reformulated as
\begin{align}
        u_i^*(x)  &= \underset{u_i \, \in \, \mathcal{U}_i}{\text{min}} \; \|u_i - u_\text{ref}^i(x)\|^2 \label{eq: qp_i} \\
        \text{s.t.} \, & \mathcal{L}_f h(x) + \mathcal{L}_g h^i(x)u_i + \mathcal{L}_g h^j(x)\omega_j + \kappa(h(x)) \geq 0, \nonumber
\end{align}
where $\mathcal{L}_gh^i : \mathbb{R}^n \rightarrow \mathbb{R}$ and $\mathcal{L}_gh^j : \mathbb{R}^n \rightarrow \mathbb{R}$ are the two components of $\mathcal{L}_g h(x) = \begin{bmatrix}\mathcal{L}_g h^i(x) & \mathcal{L}_g h^j(x)\end{bmatrix}^T$.

\subsection{Problem formulation}

\begin{defn} [Collision between two robots]
    A robot $i$ collides with another robot $j$ if $p_{ij} \notin \mathcal{X}_{ij} := \{p_{ij} : \|p_{ij}\| > r\}$, where we call $\mathcal{X}_{ij}$ the \textit{collision-free} set and $r \in \mathbb{R}^+$. 
\end{defn}

\begin{problem} [Collision avoidance for a robot swarm of unicycles with constant but different speeds following the same path $\mathcal{P}$] \label{problem}
    Suppose a set $\mathcal{N}$ of $N$ unicycles modeled by \eqref{eq: kinematics} where all robots must follow the same convex path $\mathcal{P}$. Then, given $r\in \mathbb{R}^+$ as a collision radius, and $u_\text{ref}^i$ as the control law that aligns robot $i$ with (\ref{eq: gvf}) to follow the path $\mathcal{P}$, find, if any, a supplement $u^i_\text{safe}$ to add to $u_\text{ref}^i$ such that:
    \begin{enumerate}
        \item (Safety) If $p_{ij}(0) \in \mathcal{X}_{ij}$, then $p_{ij}(t) \in \mathcal{X}_{ij}$, $\forall i,j \in\mathcal{N}, i\neq j$, and $\forall t \in \mathbb{R}^+$.
        \item (Path following) There exists a constant $\epsilon > 0$ such that the error of robot $i$ concerning the path $\mathcal{P}$ is bounded, i.e., $|e_i(t)| < \epsilon, \forall t > T$, for some $T \in \mathbb{R}^+$.
    	\item (Feasibility) It is distributed and feasible for all robots, i.e., $u^i = (u_\text{ref}^i+u^i_\text{safe})\in \mathcal{U}$, and it can be calculated by robot $i$ with only local information. 
    \end{enumerate}
\end{problem}


\section{Collision avoidance methodology}
\label{sec: collision}

\begin{figure}
\centering
\includegraphics[trim={0cm 0cm 0 0cm}, clip, width=0.8\columnwidth]{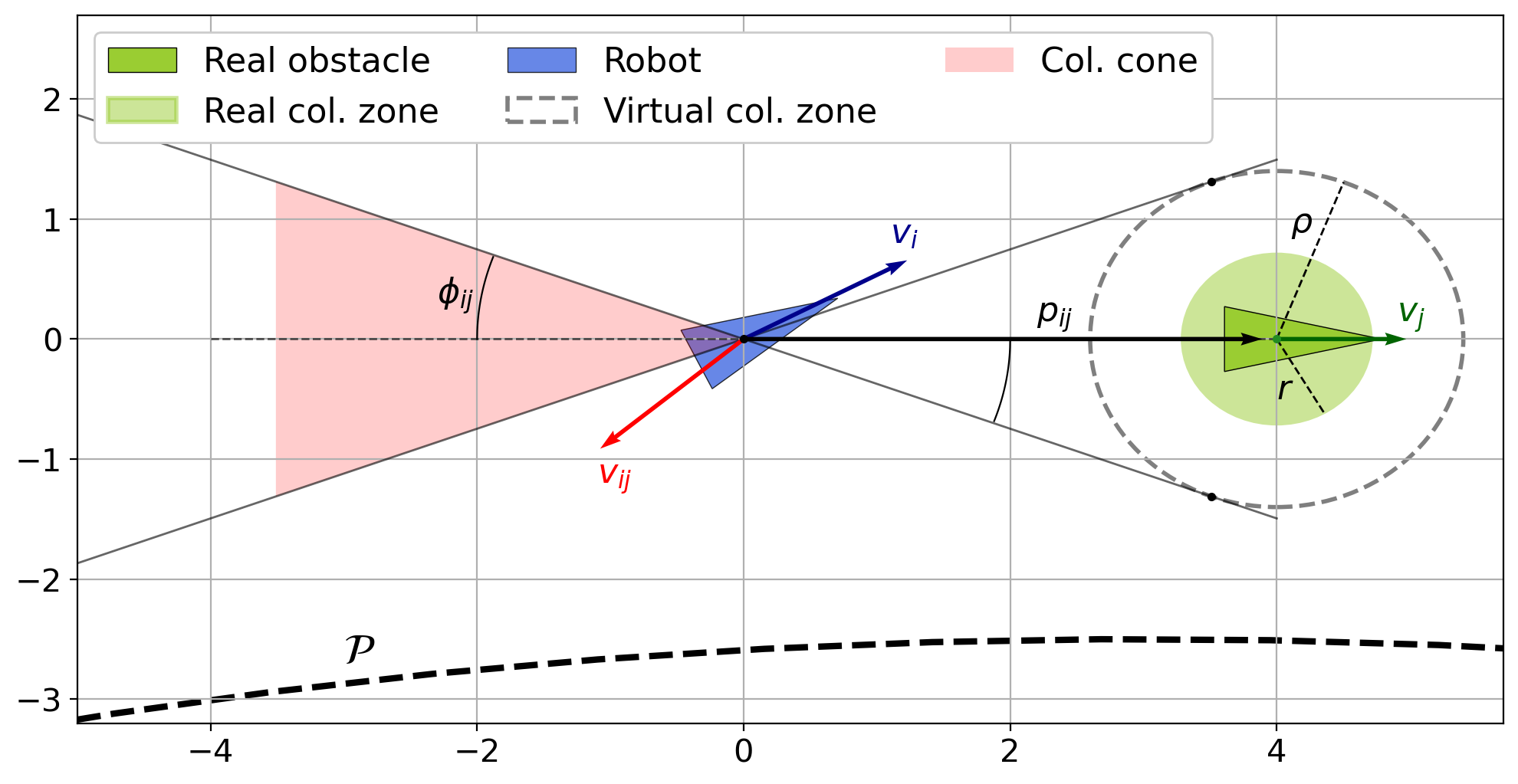}
\caption{For two non-collaborative robots, denoted as robot $i$ (blue) and robot $j$ (green), the overtaking rule in Definition \ref{def: over} dictates that the robot $i$ must keep $v_{ij}$ outside the collision cone generated by the virtual radius $\rho$ and $p_{ij}$. Here, it is important to clarify that $r\in\mathbb{R}$ represents the actual collision radius and $\rho$ assists with a smoother overtaking.
}
\label{fig: cbf_cone}
\end{figure}

For any pair of robots $(i,j)$, where robot $i$ overtakes robot $j$, and the rule to assign the roles $i$ and $j$ will be introduce shortly, we will work with the \textit{Collision Cone Control Barrier Function} (C3BF)  
\begin{equation} \label{eq: c3bf}
    h(q_{ij}(t)) = {p_{ij}}^T v_{ij} + ||p_{ij}|| \, ||v_{ij}|| \cos\phi_{ij},
\end{equation}
where $\phi_{ij} \in [0,\pi]$, given by $\scalemath{0.8}{\cos\phi_{ij} = \frac{\sqrt{\|p_{ij}\|^2 - \rho(\|p_{ij}\|)^2}}{||p_{ij}||}}$, denotes the half angle of the collision cone of the robot $i$ for the obstacle $j$, as in Figure \ref{fig: cbf_cone}. The function $\rho(\|p_{ij}\|):\mathbb{R}^+ \rightarrow \mathbb{R}^+$ is
increasing and continuous
with $\rho(r) = r$ and $\rho(\|p_{ij}\|) \leq \|p_{ij}\|$.
Thus, $\rho$ can create a large virtual collision zone while the vehicles are far away from each other, and a similar area to the real collision radius $r$ as they get closer. Although counter intuitive, this strategy modulates by design the \textit{overtaking aggressiveness}, e.g., the robot $i$ starts a smooth overtaking from enough distance rather than changing its direction aggressively once robot $j$ is \emph{very close}, i.e., it enables to minimize $|u_{safe}^{i}|$, which assists with the feasibility point in Problem \ref{problem}. 
Nonetheless, if for a given $q_{ij}$ we want $h(q_{ij})>0$, then the designed $\rho$ has to fulfill
\begin{equation} \label{eq: rho_cond}
    \sqrt{1 - \left(\frac{\rho(\|p_{ij}\|)}{\|p_{ij}\|}\right)^2} > - \hat p_{ij}^\top \hat v_{ij}.
\end{equation}
Once we present all our technical results, we will provide an example of how to design $\rho$ accordingly.

The C3BF \eqref{eq: c3bf} inherently operates on robot pairs; consequently, we denote its associated safe set as $\mathcal{C}_{ij}$. Note that if $\mathcal{C}_{ij}$ and $\mathcal{C}_{ji}$ are forward invariant, then $\mathcal{X}_{ij}$ and $\mathcal{X}_{ji}$ are also forward invariant if $\rho(\|p_{ij}\|)$ in (\ref{eq: c3bf}) is designed accordingly, i.e., collisions happen only outside the safe set. However, the C3BF cannot be used straightforwardly; inevitably, there are situations where \eqref{eq: c3bf} is not a valid candidate as a Control Barrier Function (CBF) for a swarm of unicycle robots with non-brake capabilities. For example, consider three robots where one has to avoid the other two simultaneously but the safety conditions demand $u<0$ ($\omega<0$) and $u>0$ ($\omega>0$). The existence of such non-corner cases requests from us to look for an \emph{overtaking rule} that in combination with tracking the direction of the guiding vector field (\ref{eq: gvf}) makes the robot swarm free of conflict cases in order to scale it up seamlessly. 

Before arriving at the formal overtaking rule, we need to introduce some technicalities. Let us write the safety condition \eqref{eq: psi} explicitly so that we can clearly identify its different terms. Thus, we need the calculation of the derivative of \eqref{eq: c3bf}

\begin{align} 
         \scalemath{0.8}{\dot h(q_{ij})} & 
         \scalemath{0.8}{= 
         {\dot p_{ij}}^T v_{ij} + {p_{ij}}^T \dot v_{ij} + ||v_{ij}|| \frac{ {p_{ij}}^T v_{ij} - \rho\dot\rho}
         {\sqrt{\|p_{ij}\|^2 - \rho^2}}
         + \; {v_{ij}}^T \dot v_{ij} \frac{\sqrt{\|p_{ij}\|^2 - \rho^2}}{||v_{ij}||}
         } \nonumber \\
         & \scalemath{0.8}{= 
         \underbrace{{v_{ij}}^T v_{ij} + ||v_{ij}|| 
         \frac{{p_{ij}}^T v_{ij} - \rho\dot\rho}
         {\sqrt{\|p_{ij}\|^2 - \rho^2}}}_{\mathcal{L}_f h(q_{ij})} 
         + \underbrace{{p_{ij}}^T \dot v_j + {v_{ij}}^T \dot v_{j} \frac{\sqrt{\|p_{ij}\|^2 - \rho^2}}{||v_{ij}||}}_{\mathcal{L}_g h^j(q_{ij}) u_j}} \nonumber \\
         & \scalemath{0.8}{\;\; 
         \underbrace{
         - {p_{ij}}^T \dot v_i - {v_{ij}}^T \dot v_{i} \frac{\sqrt{\|p_{ij}\|^2 - \rho^2}}{||v_{ij}||}}_{\mathcal{L}_g h^i(q_{ij}) u_i}}, \label{eq: h_dot}
\end{align}
where we can identify $L_g h^i(q_{ij}) u_i$ and $L_g h^j(q_{ij}) u_j$ by considering those terms involving $\dot v_i$ and $\dot v_j$ since they include the control laws $u_i = \omega_i$ and $u_j = \omega_j$. In particular, after some straightforward calculations, we can arrive at 
\begin{align}
        \mathcal{L}_g h^j(q_{ij}) &= \|p_{ij}\| \left( \hat p_{ij} + \hat v_{ij} \cos\phi_{ij} \right)^T (-Ev_j),
\label{eq: lghj} \\
        \mathcal{L}_g h^i(q_{ij}) &= \|p_{ij}\| \left( \hat p_{ij} + \hat v_{ij} \cos\phi_{ij} \right)^T (Ev_i).
\label{eq: lghi}
\end{align}

It is explicit now that the inner expression in the safety condition \eqref{eq: CBF_cond} consists of four added terms: three given by \eqref{eq: h_dot}, with $ \mathcal{L}_g h^j(q_{ij})$ and $\mathcal{L}_g h^i(q_{ij})$ as in \eqref{eq: lghj} and \eqref{eq: lghi} respectively, and one given by $\kappa(h(x))$ with $h(x)$ as in \eqref{eq: c3bf}. In a non-collaborative context where robot $i$ overtakes robot $j$, we note that out of the four terms, robot $i$ only has control over $L_g h^i(q_{ij})\omega_i$ to ensure that the addition of these four terms stays positive. We notice that $L_g h^i(q_{ij})u_i$ can be split as 
\begin{equation} \label{eq: lghi_split}
    L_g h^i(q_{ij})u_i = L_g h^i(q_{ij}) u_{\text{ref}}^i + L_g h^i(q_{ij}) u_{\text{safe}}^{ij},
\end{equation}
where the first term is given by the guiding vector field, and $u_\text{safe}^{ij} = - \Psi(q_{ij})/L_g h^i(q_{ij})\in\mathbb{R}$ is the action that robot $i$ has to take to avoid robot $j$, i.e., it will assist robot $i$ with making the inner expression of \eqref{eq: CBF_cond} positive by solving its QP problem \eqref{eq: qp_i} concerning the robot $j$. Now, we are ready to formulate the overtaking rule and see how it fits into the assistance of making positive the inner expression of \eqref{eq: CBF_cond}. The overtaking rule assumes that the robots orbit clockwise around $\mathcal{P}$.

\begin{defn} [Overtaking] \label{def: over}
  For a pair of non-collaborative robots $(i,j)$, the robot $i$ has the role of overtaking robot $j$. The set of robots that $i$ can overtake belongs to the set
    $\mathcal{N}_i := \{j \, :  L_gh^i(q_{ij}) > 0, \forall j\neq i\in\mathcal{N}\}.$ We say that robot $i$ is \emph{overtaking} robot $j$ if and only if $j\in\mathcal{N}_i$ and $\Psi(q_{ij}) < 0$.
\end{defn}

Since the overtaking rule requires $L_gh^i(q_{ij}) > 0$, it entails that $u_\text{safe}^{ij} > 0$ to assist with satisfying \eqref{eq: CBF_cond}, i.e., all the overtakes will be on the outside since the robots are following $\mathcal{P}$ with the same direction. Overtaking on the outside avoids robot $i$ the situation of an impossible decision of choosing between $u_\text{safe}^{ij} > 0$ and $u_\text{safe}^{ik} < 0$, with $k\neq j\in\mathcal{N}_i$. In fact, we will see that such an overtaking rule scales seamlessly for a robot swarm of unicycles, i.e., among all the \emph{supplements} $u_\text{safe}^{ij}$'s resulting from solving the $|\mathcal{N}_i|$ QP Problems (\ref{eq: qp_i}), robot $i$ only has to choose the largest supplement to turn and overtake all the robots in $\mathcal{N}_i$ safely. More importantly, we will see that such a rule is not harming the rest of the robots that need to overtake as well. Our tecnical results will require to split the overtaking into two sequential stages as in Figure \ref{fig: over}.

\begin{figure}[h!]
\centering

\begin{tikzpicture}[
    font=\scriptsize,
    shorten >=1pt, node distance = 8mm and 28mm, on grid, auto, 
    rct/.style = {draw, minimum height=7mm, minimum width=22mm},
    outer/.style={draw=gray,dashed,thick,inner sep=5pt},
    stage1/.pic={
        \node at (0,0.3) {Stage 1};
        \def\ytext{0.3}
        \fill[blue!100!white!20, opacity=0.8] (0,-\ytext) -- (-0.43,-\ytext) arc (180:0:0.43) -- cycle;
        \draw[thick,-{Latex[length=2mm]}] (0,0 -\ytext) -- (0.5,0 -\ytext) node[anchor=center] {};
        \draw[thick,-{Latex[length=2mm]}, blue] (0,0 -\ytext) -- (0.25,0.43 -\ytext) node[anchor=center] {};
        \node[blue] at (-0.08,0.33 -\ytext) {$\hat v_i$};
        \node at (0.4,0.15 -\ytext) {$\hat v_j$};
        },
    stage2/.pic={
        \node at (0,0.3) {Stage 2};
        \def\ytext{0.1}
        \fill[blue!100!white!20, opacity=0.8] (0,-\ytext) -- (-0.43,-\ytext) arc (-180:0:0.43) -- cycle;
        \draw[thick,-{Latex[length=2mm]}] (0,0 -\ytext) -- (0.5,0 -\ytext) node[anchor=center] {};
        \draw[thick,-{Latex[length=2mm]}, blue] (0,0 -\ytext) -- (0.25,-0.43 -\ytext) node[anchor=center] {};
        \node[blue] at (-0.08,-0.33 -\ytext) {$\hat v_i$};
        \node at (0.4,0.15 -\ytext) {$\hat v_j$};
        }
    ]
    
    \tikzstyle{every state}=[fill={rgb:black,1;white,10}]
    \node (q_1) [rct] {\textbf{Non-overtaking}};
    \node (q_2) [state, above right = of q_1, minimum size=1.3cm]{};
    \node (q_3) [state, right = of q_2, minimum size=1.3cm] {};

    \draw (q_2) pic {stage1};
    \draw (q_3) pic {stage2};
    
    \path[->]
        (q_1) edge [bend left]  node {O.R. \green{$\bullet$}}      (q_2)
        (q_2) edge [bend left]  node {$\hat v_j^T E \hat v_i < 0$} (q_3)
        (q_2) edge [bend left]  node {O.R. \red{$\bullet$}}        (q_1)
        (q_3) edge [bend left]  node {O.R. \red{$\bullet$}}        (q_1);

    \draw[dashed] (2,0) -- (2,2) -- (6.3,2) -- (6.3,0) -- (2,0);
    \node[text width=3cm] at (3.6,1.8) {\textbf{Overtaking}};
\end{tikzpicture}

\caption{O.R. stands for the overtaking rule in Definition \ref{def: over}. When the overtaking begins we have $\hat v_j^T E \hat v_i \geq 0$ as an initial condition.}
\label{fig: over}
\end{figure}
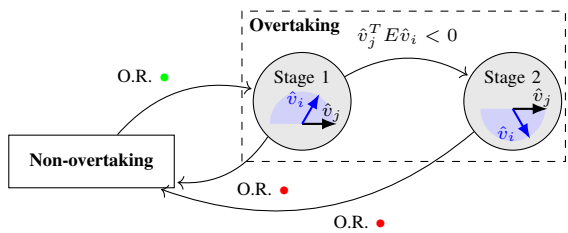 

It is worth noting that when the overtaking is not aggressive in a convex path, as the virtual collision radius $\rho$ and the guiding vector field (\ref{eq: gvf}) promote, then it is common to find $\hat p_{ij}(t)^T \, E \, \hat v_i(t) > 0$ since it entails that robot $i$ is facing at the outside regarding $p_{ij}$. Hence, the following mild assumption, check Figure \ref{fig: ass1}, is mostly necessary to discard \emph{problematic} initial conditions for the technical results. 

\begin{figure}
\centering
\includegraphics[trim={1cm 1cm 0cm 0cm}, clip, width=0.75\columnwidth]{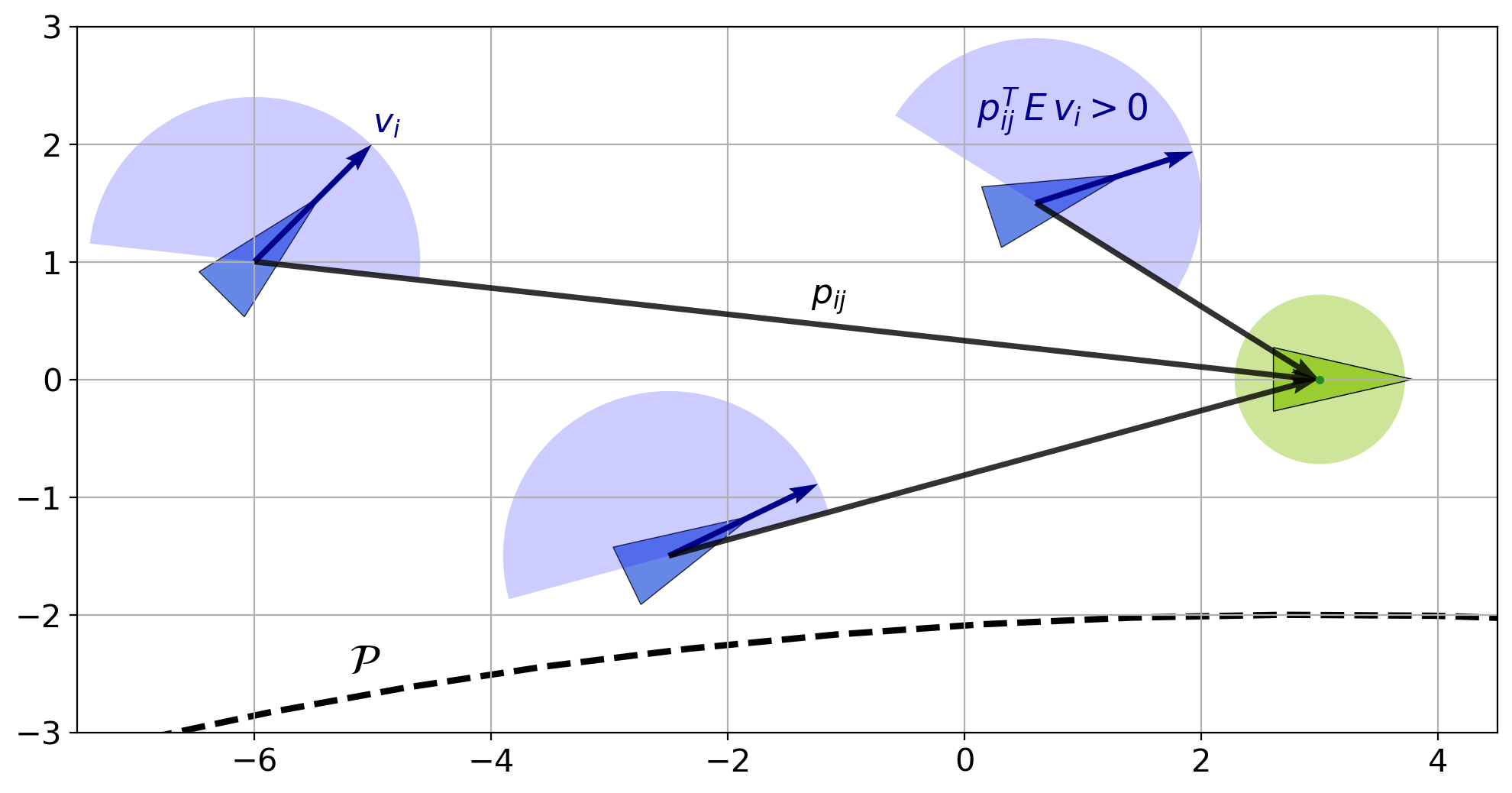}
\caption{Illustration of three robots satisfying Assumption \ref{asmp: pij_vi}. The blue semicircle indicates all the admissible directions for $v_i$ regarding $p_{ij}$.}
\label{fig: ass1}
\end{figure}
\begin{figure*}[t!]
    \centering
    \includegraphics[trim={0cm 0cm 0cm 0cm}, clip, width=1.6\columnwidth]{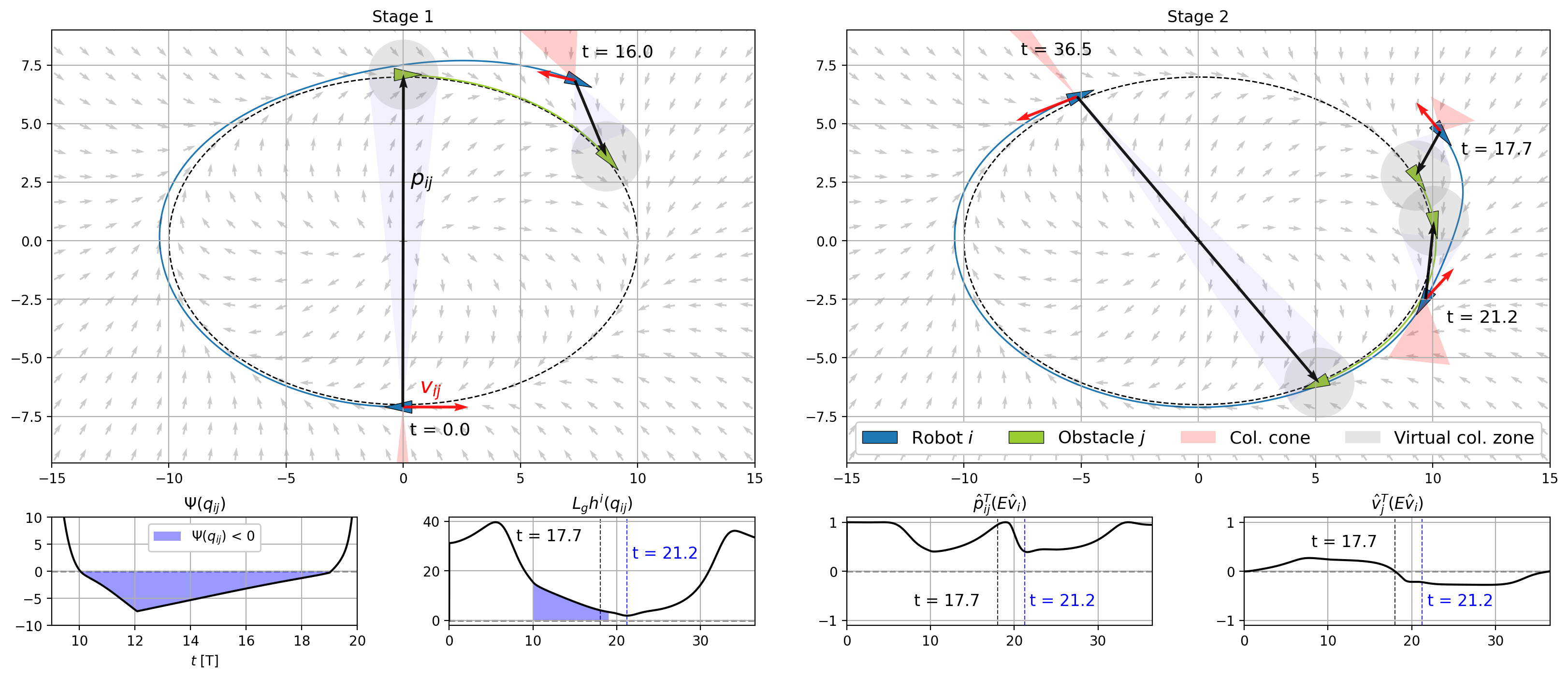}
    \caption{A pair of robots $(i,j)$ follows an elliptical path, where $s_i > s_j$ with the Definition \ref{def: over} as the overtaking rule. Since $\rho,\kappa$ are designed according to Lemma \ref{lemm: rho_kappa}, and this scenario fulfills the pre-overtaking conditions of Lemma \ref{lemm: over_man} and the Assumption \ref{asmp: pij_vi}, the simulation validates the prediction of having $L_gh^i(q_{ij})>0$ during the overtaking (blue colored area). Note that the stage changes at $t=17.7$ and $L_gh^i(q_{ij})$ reach their minimum value at $t=21.2$.}
\label{fig: stages}
\end{figure*}

\begin{assump}
\label{asmp: pij_vi}
    During the overtaking situation for the pair of robots $(i,j), i\in\mathcal{N}, j\in\mathcal{N}_i$, it holds that $\hat p_{ij}^T \, E \, \hat v_i > 0$.
\end{assump} 

\begin{lemma} 
\label{lemm: rho_kappa}
If \eqref{eq: rho_cond} holds, then there exists a $\kappa(h(q_{ij})) > - \dot h(x, u_{ref}(q_{ij}))$ such that $\Psi(q_{ij})>0$.
\end{lemma}
\begin{proof}
Substituting \eqref{eq: rho_cond} in \eqref{eq: c3bf} yields $h(q_{ij}) > 0$. Thus, since $\kappa(h(q_{ij}))$ is an (extended) class $\mathcal{K}$ function, there exists a $\kappa(h(q_{ij})) > - \dot h(x, u_{ref}(q_{ij}))$, and substituting such a $\kappa$ in \eqref{eq: psi} leads us to $\Psi(q_{ij})>0$.
\end{proof}
\begin{lemma} \label{lemm: over_man}
Let robots $i$ and $j$ with speeds $s_i \geq s_j$ following a convex path $\mathcal{P}$, where the Definition \ref{def: over} rules when there is an overtaking and Assumption \ref{asmp: pij_vi} fulfills. Consider an initial condition $q_{ij}(0)$ such that $e(p_j) = e(p_i) = 0$ and $\|\hat v_i - \hat\tau(p_i)\| = \|\hat v_j - \hat\tau(p_j)\| = 0$, with $e(p_{\{i,j\}})$ and $\hat\tau(p_{\{i,j\}})$ as in (\ref{eq: gvf}), and \eqref{eq: rho_cond} holds, then $L_g h^i(q_{ij})>0$ when $\Psi(q_{ij}) < 0$.
\end{lemma}
\begin{proof}
    Note that $L_g h^i(q_{ij})$ in \eqref{eq: lghi} consists of the sum of two terms, and the first term $\hat p_{ij}^T \, E \, \hat v_i > 0$ because of Assumption \ref{asmp: pij_vi}. Regarding the second term, we will split our analysis for the two stages of the overtaking. For the Stage1, we have that $\hat v_j^T \, E \, \hat v_i > 0$ by definition; thus, to ensure that $L_g h^i(q_{ij})>0$ during Stage1, we only need to show that its second term is positive when $\Psi(q_{ij})$ changes to negative values with $e(p_j) = e(p_i) = 0$ and $\|\hat v_i - \hat\tau(p_i)\| = \|\hat v_j - \hat\tau(p_j)\| = 0$, i.e., both robots are on and following the path. 

First, without loss of generality, let us consider a circular path $\mathcal{P}$ and the initial conditions $\hat v_i(0) = - \hat v_j(0)$, as in Figure \ref{fig: stages}, i.e., the pair of robots starts from opposite sides of the circular path. Hence, $\hat p_{ij}(0)^\top \hat v_{ij}(0) = 0 > -1$ so that there are $\kappa$ and $\rho$ such that $\Psi(q_{ij}(0)) > 0$ according to Lemma \ref{lemm: rho_kappa}. Next, considering the previous initial condition, we analyze the behavior of $\hat v_j^T \, E \, \hat v_i$ until $\Psi(q_{ij}) < 0$. Given both robots traveling over $\mathcal{P}$ clockwise, when $s_i = s_j$ they never approach each other, i.e., $\Psi(q_{ij})$ is always positive, so the claim is trivial. On the other hand, when $s_i > s_j$ we have $|u_{ref}^i| > |u_{ref}^j|$ and $u_{ref}^{\{i,j\}}<0$. Therefore, as $\hat v_j(0)^T \, E \, \hat v_i(0) = 0$ and $u_{\{i,j\}} = u_{ref}^{\{i,j\}}$, we have that
    \begin{align} \label{eq: dot_vjevi}
        \frac{\mathrm{d}}{\mathrm{dt}}(\hat v_j^T \, E \, \hat v_i) &= \hat v_j^T\hat v_i (u_j - u_i)
    \end{align}
    remains positive until $E\hat v_i = \hat v_j$. Although at this point $\hat v_j^T \, E \, \hat v_i$ begins to decrease, it can not reach zero until $\hat v_i = \hat v_j$. However, since the path is convex and both robots are on it, the only situation where $\hat v_i = \hat v_j$ is when $p_i = p_j$, 
    indicating a collision. Hence, as no collision occurs before $\Psi(q_{ij})<0$, then $\hat v_j^T \, E \, \hat v_i > 0$ until $\Psi(q_{ij})<0$. Then, we conclude that when $\Psi(q_{ij})$ switches its sign to negative, $\hat v_j^T \, E \, \hat v_i > 0$; thus $L_g h^i (q_{ij}) > 0$ during Stage1. Of course, if $\Psi \geq 0$ triggers during Stage1, the overtaking is over.
    
If $\Psi(q_{ij}) < 0$, the overtaking continues and $\hat v_j^T \, E \, \hat v_i$ switches to negative sign eventually, i.e., the overtaking transitions to Stage2. From this point onward, $L_g h^i(q_{ij})$ may get close to zero eventually. Nonetheless, since $\Psi(q_{ij})<0$, as $L_gh^i(q_{ij}) \rightarrow 0^+$ the value of $u_{safe}^{ij}$ increases the inner expression of (\ref{eq: CBF_cond}). Hence, as the time derivative of $\hat v_j^T \, E \, \hat v_i$ is given by \eqref{eq: dot_vjevi}, before $L_gh^i(q_{ij}) = 0$ can happen, the control law $u_i$ reaches a sufficiently high value to hold $u_j - u_i > 0$, which prevents $\hat v_j^T\, E \, \hat v_i$ decreasing further. Therefore, $u_{safe}^{ij}$ ensures $L_gh^i(q_{ij})>0$ during Stage2.
\end{proof}

\begin{remark}
The conservative initial conditions of Lemma \ref{lemm: over_man} ensures that $L_gh^i(q_{ij})$ is \textit{strictly} greater than zero. Thus, we can confidently extend this result by continuity arguments to overtakings where $|e(p_j)| + |e(p_i)| <\epsilon$ and $\|\hat v_i - \hat\tau(p_i)\| + \|\hat v_j - \hat\tau(p_j)\| < \delta$, with $\epsilon, \delta\in\mathbb{R}^+$ being sufficiently small.
\label{rem: cont}
\end{remark}

\begin{remark}[On the design of $\rho$ and $\kappa$]
Since $L_gh^i(q_{ij})u_i$ is split as in \eqref{eq: lghi_split}, maintaining small negative values of $\Psi(q_{ij})$ and a safe distance from the singularity of (\ref{eq: u_safe}) at $L_gh^i = 0$ is an effective strategy to keep $u_i(t) \in \mathcal{U}, \forall t$. Firstly, since $|\Psi(q_{ij})| = |\dot h(q_{ij}, u_{ref}^i) + \kappa(h(q_{ij}))|$, if $h(q_{ij}) > 0$ when $\Psi(q_{ij})<0$ we can design $\kappa$ accordingly to bound the minimum negative value of $\Psi(q_{ij})$, e.g, 
	$\kappa = \gamma h^3$, where $\gamma \in \mathbb{R}^+$ serves as a constant to modulate the safety condition's tolerance with respect to the rate of decrease in $h(q_{ij})$. However, it is essential to exercise caution when selecting high values of $\gamma$, as they increase the overtaking aggressiveness. Hence, it is necessary to design $\rho$ accordingly, e.g., $\rho(\|p_{ij}\|) = \frac{\|p_{ij}\|^d}{r^{d - 1}}$, where $d \in[0,1]$ represents a constant that modulates the rate of increase in the derivative of $\rho$ as the robot moves away from the obstacle. A higher value of $d$ implies a smooth overtake, but it also compromises \eqref{eq: rho_cond}, e.g., $d=1$ yields $\rho(\|p_{ij}\|) = \|p_{ij}\|$ and makes it impossible to satisfy \eqref{eq: rho_cond} when $\hat p_{ij}^\top \hat v_{ij} \leq 0$. Therefore, since $\rho$ should be designed so that Lemma \ref{lemm: rho_kappa} applies, given the set of expected initial conditions $q_{ij}(0)$, check Figure \ref{fig: cbf_cone}, one should choose $d$ that fulfills \eqref{eq: rho_cond} for such initial conditions. Another interesting alternative for $\rho$ is the sigmoid function.
\label{rem: rhokappa}
\end{remark}

Now, we are ready to show that the CBF (\ref{eq: c3bf}) is a valid one for a pair of robots $(i,j)$ given the Definition \ref{def: over} as overtaking, and the technical result in Lemma \ref{lemm: over_man}.

\begin{prop} \label{pro: cbf}
 Consider a pair of robots $(i,j)$ with the relative dynamics \eqref{eq: rel_kinematics}, the Definition \ref{def: over} as overtaking rule and the pre-overtaking conditions of Lemma \ref{lemm: over_man}; then, according to the Definition \ref{def: validcbf}, the CBF \eqref{eq: c3bf} is valid in $\mathcal{C}_{ij}$.
\end{prop}

\begin{proof}
Firstly, let us suppose that $q_{ij} \in \mathcal{C}_{ij}$, so we can guarantee  $q_{ij} \in \mathcal{X}_{ij}$ conservatively. Hence, as $r < \|p_{ij}\|$, then \eqref{eq: c3bf} is well-defined. Secondly, robots do not collaborate because of the Definition \ref{def: over} on overtaking; therefore, $u_i$ is the only control law that can be modified to make the inner expression of \eqref{eq: CBF_cond} positive. Without loss of generality, let us take $u_i = u_{\text{ref}}^i + u_{\text{safe}}^{ij}$ and split $L_g h^i(q_{ij}) u_i$ as in \eqref{eq: lghi_split}. Hence, as Lemma \ref{lemm: over_man} guarantees that $L_g h^i(q_{ij}) > 0$, then it is always possible to find one $u_{\text{safe}}^{ij}$ that holds \eqref{eq: CBF_cond}. Thus, the CBF \eqref{eq: c3bf} is valid in $\mathcal{C}_{ij}$.
\end{proof}

\begin{figure*}[h!]
\centering
\includegraphics[trim={2.4cm 0.2cm 1.8cm 0.8cm}, clip, width=1.26\columnwidth]{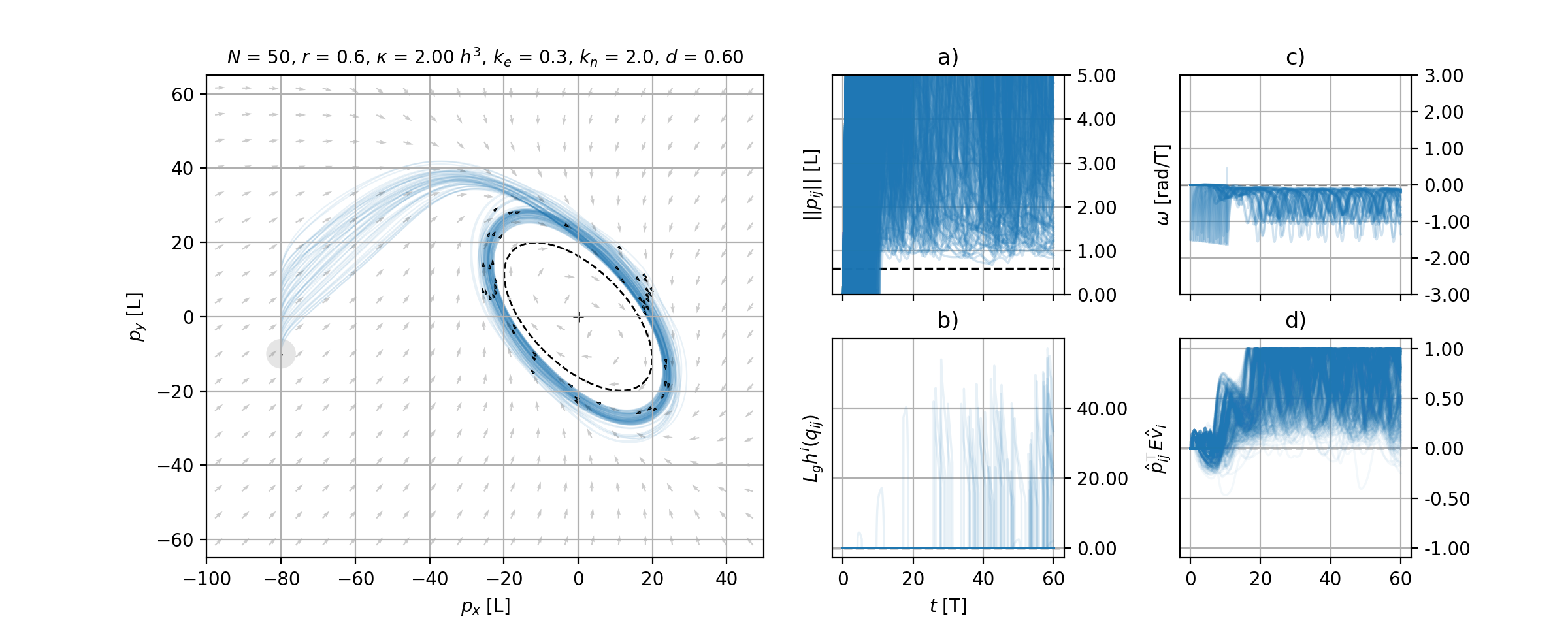}
\caption{A \emph{mother ship} travels vertically up at $X = -80$ and launches during $10$ T units a swarm of $50$ unicycles to orbit around an ellipse with different speeds of around $5$ L units per T unit. A robot is launched once the previous one is at $2.5$ times the safety distance $r$. The arrows show the direction of the GVF (\ref{eq: gvf}). Each robot $i$ is provided with the $\omega_j$ of all $j\in\mathcal{N}_i$, and $\mathcal{N}_i$ is generated by design before the simulation starts. Clusters of robots form for some periods when the speed of the robots are similar. On the right side: \textbf{a)} the distances between all the robots and the safety distance (black dashed line), during $t\in[0,10]$ many robots are together in the mother ship explaining the $0$ distance in the plot; \textbf{b)} $L_gh^i$ when $\Psi <0$, which remains positive; \textbf{c)} the proper design of $\kappa$ and $\rho$ as in Remark \ref{rem: rhokappa} promote smooth overtaking with non-aggressive $\omega$'s; \textbf{d)} once the last robot is launched, the cases when $\hat p_{ij}^\top E \hat v_{ij}\leq0$ are isolated and do not compromise the overtaking, i.e., the Assumption \ref{asmp: pij_vi} is conservative for this simulation.}
\label{fig: sim2}
\end{figure*}

\begin{figure*}[h!]
\centering
\includegraphics[trim={2cm 0.2cm 1.5cm 0.8cm}, clip, width=1.26\columnwidth]{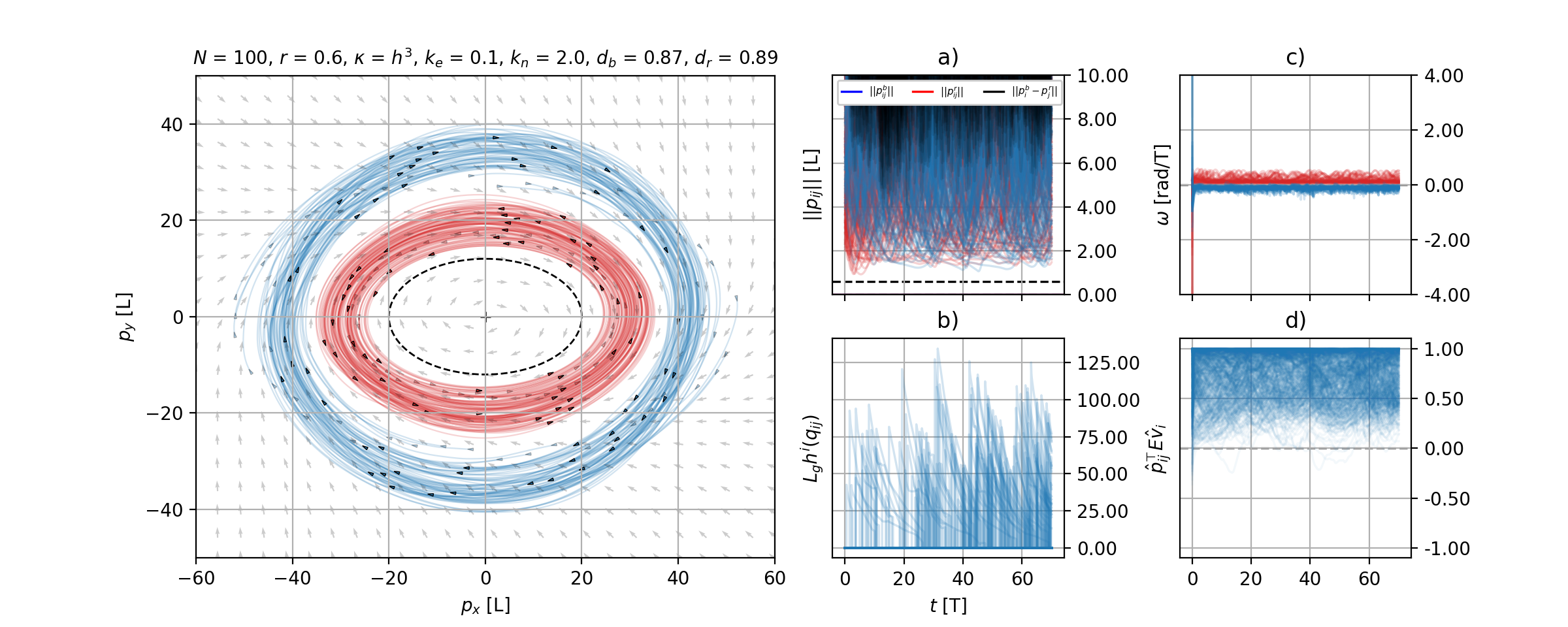}
\caption{Two robot swarms of $50$ robots each orbiting in opposite directions an elliptical path with diverse speeds between $3$ and $6$ L units per T unit. The arrows show the direction of the GVF (\ref{eq: gvf}) for the blue swarm, which turns in the opposite direction of the GVF for the red swarm. The red swarm starts deployed around a lower level set than the blue swarm. The red swarm ignores the blue one, while the blues \emph{overtake} the reds. The blue robots consider red ones as fixed obstacles for triggering the overtaking rule, i.e., $v_j = \omega_j = 0$; thus, Assumption \ref{asmp: pij_vi} is conservative and not necessary. An emergent behavior occurs with a clear boundary between the two swarms. On the right: \textbf{a)} the distances between all the robots and the safety distance (black dashed line); \textbf{b)} $L_gh^i$ when $\Psi <0$, which remains positive; \textbf{c)} the design of $\kappa$ and $\rho$ as in Remark \ref{rem: rhokappa} not only promote smooth overtaking with non-aggressive $\omega$'s but the blue swarm can avoid the red one, the spike of $\omega$ at the first instants is due to the alignment of the robots with the guiding vector field (\ref{eq: gvf}); \textbf{d)} the cases when $\hat p_{ij}^\top E \hat v_{ij}\leq0$ are isolated and do not compromise the overtaking, i.e., the Assumption \ref{asmp: pij_vi} is conservative for this simulation.}
\label{fig: sim3}
\end{figure*}

Once we have confirmed that \eqref{eq: c3bf} is a valid CBF for non-collaborative unicycle robot pairs following a convex path $\mathcal{P}$, we can scale up to a robot swarm with the collision avoidance among its individuals using the following control law
\begin{equation} \label{eq: u_safe_coop}
    u_{\text{safe}}^{i} = \max (u_{\text{safe}}^{ij} : j \in \mathcal{N}_i),
\end{equation}
where all the $u_{safe}^{ij}$'s can be computed with the C3BF-QP explicit solution given by \eqref{eq: u_safe}. The idea of this new controller is that each robot just focuses on its set $\mathcal{N}_i$, i.e., it drastically reduces the network connectivity requirements of the swarm. 
\begin{remark} \label{rem: N_i_omega_j}
    Although in \eqref{eq: u_safe_coop} $i$ requests $w_j$ for all $j \in \mathcal{N}_i$, robot $i$ can consider a worst-case scenario $\Omega_j = \operatorname{sup}\{\mathcal{U}\}$, i.e., $\Omega_j \geq \omega_j$. Hence, since $p_{ij}$ and $v_{ij}$ can be measured by $i$, then \eqref{eq: u_safe_coop} can be computed using only local information.
\end{remark}
\begin{theorem} [Circular formation control for unicycle-robot swarms]\label{thm: multi_cbf}
Given a set of unicycle robots $\mathcal{N}$ with relative dynamics \eqref{eq: rel_kinematics}, a $q_{ij}(0) \in \mathcal{C}_{ij}$ that fits the conditions of Lemma \ref{lemm: over_man}, and the Definition \ref{def: over} as the overtaking rule. 
If $\kappa$ in (\ref{eq: psi}) and $\rho$ in (\ref{eq: c3bf}) are designed according to Lemma \ref{lemm: rho_kappa}, $u_{safe}^{i}$ is given by \eqref{eq: u_safe_coop} $\forall i\in\mathcal{N}$, $u_{ref}^{i}$ aligns with (\ref{eq: gvf}) to follow a common convex path $\mathcal{P}$ in the same direction, then the control law $u_i  = u_{\text{ref}}^i + u_{\text{safe}}^i$ solves Problem \ref{problem}.
\end{theorem}
\begin{proof}
    Consider all robots $i\in\mathcal{N}$, we first note that the control law $u_i  = u_{\text{ref}}^i + u_{\text{safe}}^i$ solves the C3BF-QP problems \eqref{eq: CBF-QP} $\forall j\in\mathcal{N}_i$, and following Proposition \ref{pro: cbf} we have that $\mathcal{C}_{ij}$ is invariant if $q_{ij}(0) \in \mathcal{C}_{ij}$. Now, we check on the three objectives of Problem \ref{problem}. {\bf 1)} Since $\mathcal{C}_{ij}$ is invariant, then $q_{ij}(t) \in \mathcal{X}_{ij}, \forall t>0$. {\bf 2)} According to the overtaking rule in Definition \ref{def: over} and Lemma \ref{lemm: over_man}, the robot with the slowest $s_i$ will not overtake any robot, i.e., the slowest robot (or robots with equal slowest speeds) will converge asymptotically to $\mathcal{P}$. In parallel, from Lemma \ref{lemm: over_man} we can conclude that $\|p_{ij}\|$ is bounded during the overtaking. Consequently, for every robot $i\in\mathcal{N}$ there exist an $\epsilon > 0$ such that $|e_i(t)| < \epsilon, \forall t > T$, for some $T \in \mathbb{R}^+$. {\bf 3)} See Remark \ref{rem: rhokappa} and Remark \ref{rem: N_i_omega_j}.
\end{proof}

%
%


\section{Simulations}
\label{sec: sim}
We validate numerically our results by conducting two simulations. In Figure \ref{fig: sim2}, we explain how a robot swarm is \emph{launched} from a \emph{mother ship} to enclose a common elliptical path. In Figure \ref{fig: sim3}, we explain how two robot swarms orbit the same elliptical path with opposite directions. Note that the initial conditions of both simulations are not as conservative as Lemma \ref{lemm: over_man} requires, which agrees with Remark \ref{rem: cont}.



\section{Conclusion}
\label{sec: con}
We have presented a behavioral-based circular formation control for a robot swarm of unicycles with constant but different speeds. We have shown that the combination of a guiding vector field with a C3BF can guarantee that the \emph{overtaking on the outside} is scalable and safe under a mild assumption that simulations show it is not necessary.

\bibliographystyle{IEEEtran}
\bibliography{biblio}

\end{document}